%% file: document.tex
\documentclass[]{llncs}

\long\def\appRef#1{Appendix~\ref{#1}}
\long\def\techrepA#1{}

\def\fakesub#1{\paragraph{#1}}

\usepackage{amsmath,amssymb,graphicx,multirow,mathabx,mathtools,shuffle}
\usepackage{booktabs}
\usepackage[british]{babel}

\usepackage{tikz}
\usepackage{tikz-qtree}
\usetikzlibrary{positioning}
\usepackage{hyperref}

\usetikzlibrary{fit,calc}

\usepackage[makeroom]{cancel}
\usepackage{subcaption}
\usepackage{algorithmicx,algpseudocode,wrapfig}
\changenotsign

\DeclareMathOperator{\op}{\oplus}
\DeclareMathOperator{\xorOp}{\times}
\DeclareMathOperator{\sequenceOp}{\rightarrow}
\DeclareMathOperator{\loopOp}{\circlearrowleft}
\DeclareMathOperator{\concurrentOp}{\land}
\DeclareMathOperator{\interleavedOp}{\leftrightarrow}
\DeclareMathOperator{\orOp}{\lor}
\newcommand{\lang}[1]{\mathcal{L}(#1)}

	\newcommand{\ruleSin}{\textsc{S}}
	\newcommand{\ruleAssXor}{\textsc{A}\ensuremath{_{\xorOp}}}
	\newcommand{\ruleAssSeq}{\textsc{A}\ensuremath{_{\sequenceOp}}}
	\newcommand{\ruleAssCon}{\textsc{A}\ensuremath{_{\concurrentOp}}}
	\newcommand{\ruleAssOr}{\textsc{A}\ensuremath{_{\orOp}}}
	\newcommand{\ruleAssLoopB}{\textsc{A}\ensuremath{_{\loopOp\textsc{b}}}}
	\newcommand{\ruleAssLoopR}{\textsc{A}\ensuremath{_{\loopOp\textsc{r}}}}
	\newcommand{\ruleTauXor}{\textsc{T}\ensuremath{_{\xorOp}}}
	\newcommand{\ruleTauSeq}{\textsc{T}\ensuremath{_{\sequenceOp}}}
	\newcommand{\ruleTauCon}{\textsc{T}\ensuremath{_{\concurrentOp}}}
	\newcommand{\ruleTauInt}{\textsc{T}\ensuremath{_{\interleavedOp}}}
	\newcommand{\ruleTauOr}{\textsc{T}\ensuremath{_{\orOp}}}
	\newcommand{\ruleTauOrXor}{\textsc{T}\ensuremath{_{\orOp,\xorOp}}}
	\newcommand{\ruleTauLoopB}{\textsc{T}\ensuremath{_{\loopOp\textsc{b}}}}
	\newcommand{\ruleTauLoopR}{\textsc{T}\ensuremath{_{\loopOp\textsc{r}}}}
	\newcommand{\ruleTauLoopBR}{\textsc{T}\ensuremath{_{\loopOp\textsc{br}}}}
	\newcommand{\ruleConInt}{\textsc{C}\ensuremath{_{\interleavedOp}}}
	\newcommand{\ruleConOr}{\textsc{C}\ensuremath{_{\orOp}}}
	
\newcommand{\inlineTree}[2]{\tikz[baseline=(root.base),level distance=0.7cm, sibling distance=0.1cm, inner sep=0.01cm]{%
		\tikzset{execute at begin node=\strut}%
		\Tree%
		[.\node (root) {#1}; #2 ]}} 
        
\usepackage{enumitem,xparse}
    \makeatletter

        \newcounter{labelledEnumerateCounter}

            \def\storelength#1#2#3{%
                \immediate\write\@auxout{\string\global\string\def\string\oldLabelledEnumerateList#3{#2}}%
            }
            \def\getlength#1#2{
                \ifcsname #1\roman{#2}\endcsname%
                    \csname #1\roman{#2}\endcsname%
                \else%
                    0pt%
                \fi%
            }

            \def\getWidth#1{
                \ifcsname oldLabelledEnumerateWidth#1\endcsname%
                    \csname oldLabelledEnumerateWidth#1\endcsname%
                \else%
                    1pt%
                \fi%
            }
            \def\setWidth#1#2{
                \global\expandafter\edef\csname oldLabelledEnumerateWidth#1\endcsname{\the#2}%
            }

        \newlength{\widthNow}
        \newlength{\widthThis}

        \NewDocumentEnvironment{labelledEnumerate}{}{%
            \stepcounter{labelledEnumerateCounter}
            \begingroup

            \edef\localCounter{\roman{labelledEnumerateCounter}}

            \begin{enumerate}[style=standard,align=left,labelwidth=\getlength{oldLabelledEnumerateList}{labelledEnumerateCounter}]%

            \let\oldMakeLabel\makelabel
            \let\makelabel\mymakelabel
        }{%
            \end{enumerate}%
            \storelength{\oldLabelledEnumerateStack}{\getWidth{\localCounter}}{\localCounter}%
            \endgroup
        }

        \newcounter{namedEnumerateCounter}

        \NewDocumentEnvironment{namedEnumerate}{mmm}{%
            \stepcounter{labelledEnumerateCounter}
            \begingroup

            \edef\localCounter{\roman{labelledEnumerateCounter}}

            \index{{#1}.1-#1.\ref{last:\thenamedEnumerateCounter}!#3}%
            \begin{enumerate}[label=#1.\arabic*,ref=\mbox{#2#1.\arabic*},labelindent=0pt,labelwidth=\getlength{oldLabelledEnumerateList}{labelledEnumerateCounter},itemindent=0em,leftmargin=!]%

            \let\oldMakeLabel\makelabel
            \let\makelabel\mymakelabel
        }{%
            \edef\@currentlabel{\arabic{\@enumctr}}%
            \label{last:\thenamedEnumerateCounter}%
            \stepcounter{namedEnumerateCounter}%

            \end{enumerate}%
            \storelength{\oldLabelledEnumerateStack}{\getWidth{\localCounter}}{\localCounter}%
            \endgroup
        }
    \makeatother

\title{Language-Preserving Reduction Rules for Block-Structured Workflow Nets}
\titlerunning{Reduction Rules for Block-Structured Workflow Nets}
\author{Sander J.J. Leemans}
\institute{RWTH University, Aachen, Germany\\sander.leemans@rwth-aachen.de}

\begin{document}

\setlength{\abovedisplayskip}{0pt}
\setlength{\belowdisplayskip}{0pt}
\setlength{\abovedisplayshortskip}{0pt}
\setlength{\belowdisplayshortskip}{0pt}
\setlength{\itemsep}{0pt}

	\maketitle

	\begin{abstract}
		Process models are used by human analysts to model and analyse behaviour, and by machines to verify properties such as soundness, liveness or other reachability properties, and to compare their expressed behaviour with recorded behaviour within business processes of organisations.
		For both human and machine use, small models are preferable over large and complex models: for ease of human understanding and to reduce the time spent by machines in state space explorations.
		Reduction rules that preserve the behaviour of models have been defined for Petri nets, however in this paper we show that a subclass of Petri nets returned by process discovery techniques, that is, block-structured workflow nets, can be further reduced by considering their block structure in process trees.
		We revisit an existing set of reduction rules for process trees and show that the rules are correct, terminating, confluent and complete, and for which classes of process trees they are and are not complete.
		In a real-life experiment, we show that these rules can reduce process models discovered from real-life event logs further compared with rules that consider only Petri net structures.
	\end{abstract}
	
	
	
	\section{Introduction}
	\label{sec:introduction}
		\input{sec1intro}

	\section{Related Work}
	\label{sec:relatedwork}
		\input{sec2relatedwork}
   
	\section{Preliminaries}
	\label{sec:preliminaries}
		\input{sec3preliminaries}
	
	
	\section{Reduction Rules for Process Trees}	
	\label{sec:rules}
		\input{sec4rules}
	
	\section{Motivation \& Analysis}
	\label{sec:algorithm}
		\input{sec5processtrees}
	
	\section{Evaluation}
	\label{sec:evaluation}
		\input{sec6evaluation} 
	
	\section{Conclusion}
	\label{sec:conclusion}
		\input{sec7conclusion}
    	
	\bibliographystyle{splncs03}
	\bibliography{representaties0}
	
	\appendix
	
	\section{Function: Does a Process Tree Express the Empty Trace}
	\label{app:treeEmptyTrace}
		
		\begin{align*}
			\epsilon \in \lang{a} ={}& \text{false}\\
			\epsilon \in \lang{\tau} ={}& \text{true}\\
			\epsilon \in \lang{\xorOp(M_1, \ldots M_n)} ={}& \epsilon \in \lang{M_1} \lor \ldots \epsilon \in \lang{M_n}\\
			\epsilon \in \lang{\sequenceOp(M_1, \ldots M_n)} ={}& \epsilon \in \lang{M_1} \land \ldots \epsilon \in \lang{M_n}\\
			\epsilon \in \lang{\interleavedOp(M_1, \ldots M_n)} ={}& \epsilon \in \lang{M_1} \land \ldots \epsilon \in \lang{M_n}\\
			\epsilon \in \lang{\concurrentOp(M_1, \ldots M_n)} ={}& \epsilon \in \lang{M_1} \land \ldots \epsilon \in \lang{M_n}\\
			\epsilon \in \lang{\orOp(M_1, \ldots M_n)} ={}& \epsilon \in \lang{M_1} \lor \ldots \epsilon \in \lang{M_n}\\
			\epsilon \in \lang{\loopOp(M_1, \ldots M_n)} ={}& \epsilon \in \lang{M_1}
		\end{align*}
		
	\section{Function: Maximum Trace Length of Process Tree}
	\label{app:shortTraces}
		\begin{align*}
			m(\tau) ={}& 0\\
			m(a) ={}& 1\\
			m(\xorOp(M_1, \ldots M_n)) ={}& \max(m(M_1), \ldots m(M_n))\\
			m(\sequenceOp(M_1, \ldots M_n)) ={}& m(M_1) + \ldots m(M_n)\\
			m(\interleavedOp(M_1, \ldots M_n)) ={}& m(M_1) + \ldots m(M_n)\\
			m(\concurrentOp(M_1, \ldots M_n)) ={}& m(M_1) + \ldots m(M_n)\\
			m(\orOp(M_1, \ldots M_n)) ={}& m(M_1) + \ldots m(M_n)\\
			m(\loopOp(M_1, \ldots M_n)) ={}& \begin{cases} \infty & \max(m(M_1), \ldots m(M_n)) \geq 1 \\ 0 & \text{otherwise} \end{cases}
		\end{align*}
		
		\section{Class of Process Trees for Which Completeness Holds}
		\label{app:completeness}
		\def\rediscoverableClassCoo{C}		
		Adapted from~\cite{kais}.		
		
		\begin{definition}
			Let $\Sigma$ be an alphabet of activities, then the following process trees are in $\rediscoverableClassCoo{}$:
				\begin{itemize}
					\item $\tau$ is in $\rediscoverableClassCoo{}$;
					\item $a$ with $a \in \Sigma$ is in \rediscoverableClassCoo{};
					\item Let $M_1 \ldots M_n$ be reduced process trees in $\rediscoverableClassCoo{}$ without duplicate activities:
						$\forall_{i \in [1\ldots n], i \neq j \in [1\ldots n]} \Sigma(M_i) \cap \Sigma(M_j) = \emptyset$.
						Then,
						\begin{itemize}
							\item A node $\op(M_1, \ldots M_n)$ with $\op \in \{\xorOp, \sequenceOp, \concurrentOp, \orOp\}$ is in $\rediscoverableClassCoo{}$
							\item An interleaved node $\interleavedOp(M_1, \ldots, M_n)$ is in $\rediscoverableClassCoo{}$ if all:
								\begin{namedEnumerate}{i}{$\rediscoverableClassCoo$.}{}
									\item At least one child has disjoint start and end activities:
										$\exists_{i \in [1\ldots n]} \text{Start}(M_i) \cap \text{End}(M_i) = \emptyset$
									\item No child is interleaved itself:
										$\forall_{i\in[1\ldots n]} M_i \neq \interleavedOp(\ldots)$
									\item \label{req:fundamentals:oi:rediscoverableClassCoo:noOptionalUnderInterleaved} No child is optionally interleaved:
										$\forall_{i \in [1\ldots n]} M_i \neq \xorOp(\tau, \interleavedOp(\ldots)) $
									\item Each concurrent or inclusive choice child has at least one child with disjoint start and end activities:
										$\forall_{i \in [1\ldots n]} M_i = \op(M'_1, \ldots M'_m)  \Rightarrow \exists_{j \in [1 \ldots m]} \text{Start}(M'_j) \cap \text{End}(M'_j) = \emptyset $ with $\op \in \{\concurrentOp, \orOp\}$
								\end{namedEnumerate}
							\item A loop node $\loopOp(M_1, \ldots M_n)$ is in $\rediscoverableClassCoo{}$ if all:
								\begin{namedEnumerate}{l}{$\rediscoverableClassCoo$.}{}
									\item The body child is not concurrent:
										$M_1 \neq \concurrentOp(\ldots)$
									\item \label{req:fundamentals:oi:rediscoverableClassCoo:noEpsilonInLoopBody} No redo child can produce the empty trace:
										$\forall_{i \in [2\ldots n]} \epsilon \notin \lang{M_i}$
								\end{namedEnumerate}
						\end{itemize}
				\end{itemize}
			\end{definition}
				    	
\end{document}

%% file: sec1intro.tex
Process models find many uses in organisations nowadays, for instance to map, verify, measure and compare business processes within the organisation, as well as to check compliance with rules and regulations.
Process mining aims to obtain insights from recorded event data about business process executions, stored in event logs.
For instance, conformance checking techniques compare two process models or a process model with an event log of recorded behaviour in the organisation's process, thereby highlighting differences between the model and the actual behaviour  recorded in the event log, or differences between models representing, for instance, different implementations of a process in different geographical regions~\cite{DBLP:conf/caise/EckLLA15}.
Studying these differences may lead to insights into deviations, compliance and performance.

Recently, due to progress in information systems, event recording capabilities have been improved, such that larger and more complex processes can be studied~\cite{DBLP:journals/sosym/LeemansFA18}.
However, such larger and more complex process models have been shown to be more difficult to understand by humans~\cite{DBLP:conf/ifip8-1/SchrepferWMR09}, thus complicating human analysis.
Furthermore, such models challenge automated techniques that are based on state-space explorations: complex models with large state spaces might take an exponential longer time  to process than simpler models with smaller state spaces~\cite{DBLP:books/sp/CarmonaDSW18}.

In particular, \emph{silent steps}, represented by silent transitions in Petri nets~\cite{DBLP:conf/bpm/PedroC16} or $\tau$-leaves in abstract hierarchical views of block-structured workflow nets~\cite{DBLP:books/sp/Aalst16}, can be a source of unnecessary complexity and expand the state space.
Superfluous silent steps might result from process discovery techniques, which construct a process model from an event log, or from process model \emph{projection}, which may be applied by analysts or conformance checking techniques to focus the behaviour of the model on some of its activities.
For instance, the Projected Conformance Checking framework~\cite{DBLP:journals/sosym/LeemansFA18} uses projection to tuple-wise compare the behaviour of a process model with an event log based on all subsets of activities of a certain size.

In this paper, we present an approach to simplify process models, in particular abstract hierarchical representations of block-structured workflow nets (\emph{process trees})~\cite{DBLP:journals/sosym/LeemansFA18}.
Such models result from process discovery techniques \cite{DBLP:conf/cec/BuijsDA12,DBLP:conf/bpm/LeemansFA13,DBLP:conf/otm/LeemansTH18} or can be derived from Petri nets~\cite{DBLP:conf/wsfm/PolyvyanyyVV10,DBLP:journals/dke/VanhataloVK09}.
That is, we present a set of reduction rules that \emph{reduce} the structural complexity of process models while keeping the visible behaviour of the models, that is, their \emph{language}, the same.
Each rule applies to a certain structure in a part of the model; when applied, the rule reduces this part to a smaller part.
Applying these rules has the potential to make models as well as their state spaces smaller.

As all rules are based on the structure of the model rather than on its behaviour, applying the rules repeatedly has the potential of being much less time-consuming than exploring the state space of the full unreduced model.
Hence, our approach can be used as a time-saving pre-processing step to many automated techniques that use process models, or as post-processing steps after process discovery techniques.
Furthermore, by the reduced complexity of the models, they might also be better understandable by human analysis.

To maximise the applicability of the reduction rules, we identified four desirable formal properties that any set of reduction rules must ideally adhere to: 
(1) The set of rules must be \emph{correct}, that is, the application of each rule must preserve the language of the process model that the rule is applied to. 
(2) The set of rules must be \emph{terminating}, that is, exhaustive repeated application of the rules to any process model must be terminating. 
(3) The set of rules must be \emph{confluent}, that is, if multiple rules apply to a process model, then the order in which these rules are applied must not matter: an exhaustive repeated application of the rules should always result in an isomorphic model~\cite{newman1942theories}.
(4) The set of rules must be \emph{complete} with respect to a particular class of models, that is, two exhaustively reduced process models of the class must have the same language if and only if they are syntactically equivalent~\cite{kais}.
This class of models must be as large as possible.

Reduction rules for Petri nets that preserve language have been proposed before~\cite{DBLP:conf/bpm/PedroC16}, however these rules have not been formally defined and correctness, termination, confluence and completeness have not been shown.
Nevertheless, in our evaluation (Section~\ref{sec:evaluation}), we will show that the best results can be obtained by applying both the rules proposed in~\cite{DBLP:conf/bpm/PedroC16} as well as the rules proposed in this paper.
Furthermore, we show that the rules of this paper are correct, terminating, confluent and complete, and provide an implementation of both sets of rules in the ProM framework~\cite{DBLP:conf/apn/DongenMVWA05}.

The rules used in this paper were published as formal definitions in an appendix of~\cite{kais}.
This paper extends the explanation of these rules, and analyses and evaluates them. In particular, \emph{this paper contributes}:
\begin{itemize}
	\item A detailed categorisation, motivation and explanation of the rules (Section~\ref{sec:motivation});
	\item Formal proofs of their correctness, termination, confluence and completeness (Section~\ref{sec:rediscoverability});
	\item An exploration of the boundaries of the class of process trees for which the rules are complete (Section~\ref{sec:rediscoverability});
	\item A study of the differences between the rules and an existing set of rules~\cite{DBLP:conf/bpm/PedroC16} (Section~\ref{sec:rediscoverability});
	\item A real-life evaluation of the rules, that is, a comparison with an existing set of rules~\cite{DBLP:conf/bpm/PedroC16} (Section~\ref{sec:evaluation}).
\end{itemize}

The contributions of this paper are revised and extended from a PhD thesis~\cite{leemans2017robust}.
In the remainder of this paper, we first explore related work in Section~\ref{sec:relatedwork}, introduce concepts (Section~\ref{sec:preliminaries}) and restate the reduction rules (Section~\ref{sec:rules}).
Second, we motivate the rules and analyse their formal properties in Section~\ref{sec:algorithm}.
Third, we evaluate the new rules in Section~\ref{sec:evaluation}.
Finally, the paper is concluded in Section~\ref{sec:conclusion}.

%% file: sec2relatedwork.tex
Several sets of reduction rules for Petri nets have been proposed before.
For instance, in~\cite{murata1989petri}, 6 reduction rules were proposed that preserve liveness (whether the net can deadlock), safeness (whether the net never puts more than one token in a place) and boundedness (whether the net can only put a limited number of tokens in each place).
These rules have been applied as a pre-processing step in the Woflan tool~\cite{DBLP:journals/cj/VerbeekBA01}, which uses these rules to reduce the net before further soundness checks are performed.
Similarly, reduction rules have been used in the context of Ada-programs~\cite{DBLP:journals/tpds/ShatzTMD96,DBLP:journals/tse/MurataSS89} to assess concurrent programs for soundness issues, and in the context of timed Petri nets~\cite{DBLP:journals/acta/SloanB96} to assess soundness and timing properties.
Furthermore, property-preserving reduction rules have been introduced for live and safe or bounded free-choice Petri nets~\cite{DBLP:conf/concur/Desel90,DBLP:conf/apn/EsparzaS90}, and for coloured Petri nets~\cite{DBLP:conf/apn/Haddad88}.

In~\cite{DBLP:conf/apn/Berthelot85} transformations are studied to make Petri nets smaller while preserving properties like boundedness, safeness, S-invariants (whether a weighted set of places always contains the same number of tokens), proper termination (a given final marking can always be reached), presence of home states (a particular marking is reachable from every reachable marking), unavoidable markings, and liveness.

The LoLa tool~\cite{DBLP:conf/apn/Wolf07} can verify whether certain properties hold for Petri nets and give counterexamples, such as CTL and LTL formulae.
It traverses the state space of the Petri net after applying reduction techniques such as stubborn sets~\cite{DBLP:conf/apn/Valmari89}, symmetry reduction, coverability graph reduction, the sweep-line method, invariant calculus and bit hashing.
Depending on the property to be checked, reduction techniques are in- or excluded, and the reduction techniques are reordered.

Another reduction technique for soundness-related properties is to transform a process model into a Refined Process Structure Tree (RPST)~\cite{DBLP:journals/dke/VanhataloVK09,DBLP:conf/wsfm/PolyvyanyyVV10}, which constructs a hierarchy of single-entry single-exit fragments to the extent possible.
Consequently, the structure of RPSTs is exploited to verify soundness: the entire model is sound if and only if all the fragments are sound, and the fragments are smaller and thus easier to check than the full model~\cite{DBLP:journals/dke/FahlandFKLVW11}.
Similarly, in~\cite{DBLP:conf/caise/DongenAV05,DBLP:journals/topnoc/LohmannVD09} models of various formalisms are reduced before being translated to Petri nets.
Due to the similarity between RPSTs and the hierarchical block-structured Petri nets used in this paper, general Petri nets could be reduced by first translating them to RPSTs, second reducing them with reduction rules, and finally translating them to Petri nets again.
The method presented in~\cite{DBLP:journals/tsmc/LeeF85,DBLP:journals/tsmc/Lee-KwangF89} might be used in a similar fashion.

However, all these reduction techniques aim to preserve certain properties of Petri nets, but not their \emph{language}.
The first exception is the stubborn sets technique~\cite{DBLP:conf/apn/Valmari89}, which, under some conditions, can reduce process models while preserving their language.
The second exception was presented in~\cite{DBLP:conf/bpm/PedroC16}, in which a set of language-preserving reduction rules for Petri nets is proposed, which is based on the set of rules of~\cite{murata1989petri}.
These rules target redundant places and silent transitions, which are removed to make the net smaller.
The third exception is~\cite{DBLP:conf/IEEEicci/GruhnL09}, in which BPMN models are reduced to make them better understandable.
The reduction rules considered in this paper treat behaviour on a higher level: process trees.
These sets~\cite{DBLP:conf/bpm/PedroC16,DBLP:conf/IEEEicci/GruhnL09} have neither been formally defined nor has correctness, termination, confluence and completeness been shown.

In our evaluation (Section~\ref{sec:evaluation}), we will show that for process trees discovered from real-life event logs it is often beneficial to first apply the reduction rules of this paper, then translate the tree to a Petri net and finally apply the Petri net reduction rules of~\cite{DBLP:conf/bpm/PedroC16}.
This combination of process tree and Petri net reduction techniques gave the smallest models in the majority of tested cases.

%% file: sec3preliminaries.tex
Given an alphabet $\Sigma$ of process steps (activities), a \emph{Petri net} is a tuple $(P, T, f, l)$ in which $P$ is a set of places, $T$ is a set of transitions, $f \subseteq (P \times T) \cup (T \times P)$ is a flow relation and $l \colon T \rightarrow \Sigma \cup \{\tau\}$ is a labelling function, such that $S \cap T = \emptyset$ and $\tau \notin \Sigma$.
A \emph{marking} is a multiset of places, indicating the state of the net, indicating \emph{tokens}.
A transition $t$ can fire in a marking $M$ if $\forall_{(p, t) \in f} t \in M$.
When $t$ fires, then $M$ is updated according to $f$, and if $l(t) = a, a \in \Sigma$, activity $a$ is said to be executed.
Otherwise, if $l(t) = \tau$, i.e. $t$ is a \emph{silent transition}, no activity is executed.
The sequence of activities (without $\tau$s) corresponding to a sequence of transition firings that brings the net from a given \emph{initial} to a \emph{final} marking is a \emph{trace}~\cite{DBLP:conf/bpm/PedroC16}.

A \emph{workflow net} is a Petri net in which there is one place without incoming transitions (\emph{source}), there is one place without outgoing transitions (\emph{sink}), and all other elements are on directed paths from source to sink.
The initial and final markings have a token on respectively the source or the sink~\cite{DBLP:journals/tkde/AalstWM04}.
Figure~\ref{fig:workflow} shows an example.
A workflow net is \emph{sound} if each transition can be fired and, from each reachable marking, the final marking is reachable.
Sound models are hence free of deadlocks and other anomalies.

\begin{figure}
	\centering
	\input{workflow1}
	\caption{A block-structured workflow net. The dashed rectangles denote the blocks of the corresponding process tree.
		}
	\label{fig:workflow}
\end{figure}
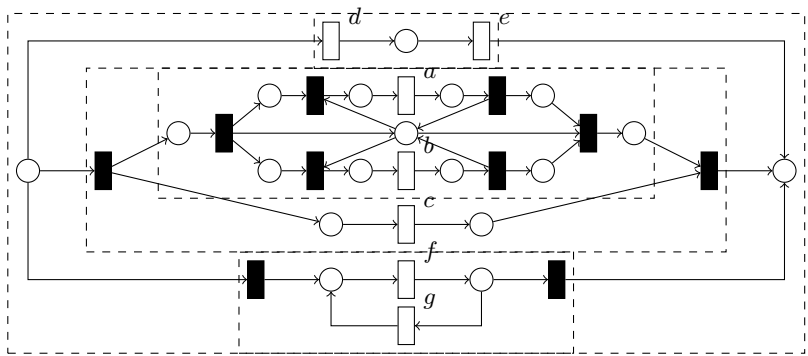

A \emph{block-structured} workflow net is a workflow net that can be hierarchically divided into parts with a single entry and a single exit. Figure~\ref{fig:workflow} contains a block-structured workflow net, and the hierarchical blocks have been annotated with dashed rectangles.
A \emph{process tree} is an abstract view on (some) block-structured workflow nets, consisting of \emph{nodes} and \emph{leaves}, expressing a language.
Formally: $a \in \Sigma$ is a process tree \emph{leaf} expressing the language $\{ \langle a \rangle \}$; $\tau \notin \Sigma$ is a process tree \emph{leaf} expressing the language with the empty trace $\{\epsilon\}$; and for process trees $S_1 \ldots S_n$ and $\op$ a process tree operator, $\op(S_1, \ldots S_n)$ is a process tree \emph{node}, expressing a combination of the languages of its children $S_1 \ldots S_n$, depending on the operator $\op$.
We consider six process tree operators $\op$, whose language $\lang{\op(S_1, \ldots S_n)}$ is defined as follows for different $\op$:
\begin{itemize}
	\item[$\xorOp$] denotes the exclusive choice between the children: 
	$$\lang{\xorOp(S_1, \ldots S_n)} = \bigcup_{1\leq i \leq n} \lang{S_i}$$
	\item[$\sequenceOp$] denotes the sequential composition of the children: 
	$$\lang{\sequenceOp(S_1, \ldots S_n)} = \lang{S_1} \cdot \lang{S_2} \cdots \lang{S_n}$$
	\item[$\interleavedOp$] denotes the interleaved composition of the children: all children must execute, but their execution cannot overlap: 
	$$\lang{\interleavedOp(S_1, \ldots S_n)} = \bigcup_{i_1 \ldots i_n\in p(n)} \lang{\sequenceOp(S_{i_1}, \ldots S_{i_n})}$$
	\item[$\concurrentOp$] denotes the concurrent composition of the children: all children must execute and their execution may overlap: 
	$$\lang{\concurrentOp(S_1, \ldots S_n)} = \lang{S_1} \shuffle \lang{S_2} \shuffle \ldots \lang{S_n}$$
	\item[$\orOp$] denotes the inclusive choice between the children: at least one child must execute, and children may overlap: 
	$$\lang{\orOp(S_1, \ldots S_n)} = \bigcup_{i_1 \ldots i_m\in q(n)} \lang{\concurrentOp(S_{i_1}, \ldots S_{i_m})}$$
	\item[$\loopOp$] denotes the execution of the first child (the \emph{body} part), followed by the optional repeated execution of a non-first child (a \emph{redo} part) and the first child: 
	$$\lang{\loopOp(S_1, \ldots S_n)} = \lang{S_1} \cdot ( \lang{\xorOp(\lang{S_2}, \ldots \lang{S_n})} \cdot \lang{S_1})^*$$
\end{itemize}
In which $p(n)$ denotes the set of all permutations of the numbers $1 \ldots n$, $\shuffle$ is a language shuffle operator~\cite{kais} and $q(n)$ is the set of all subsets of the numbers $1 \ldots n$.
For instance, consider $M_e = \xorOp(\concurrentOp(\interleavedOp(a, b), c), \sequenceOp(d, e), \loopOp(f, g))$.
Then, $\lang{M_e}$ is $\{ 
\langle a, b, c \rangle$, $\langle a, c, b \rangle$, $\langle c, a, b \rangle$, $\langle c, b, a \rangle$, $\langle b, a, c \rangle$, $\langle b, c, a \rangle$,
$\langle d, e \rangle$, 
$\langle f \rangle$, $\langle f, g, f \rangle$, $\langle f, g, f, g, f\rangle$, $\ldots \}$.
A process tree can be translated to a workflow net~\cite{leemans2017robust}, and for the operators considered in this paper, this workflow net is sound by construction.
For instance, $M_e$ can be translated to the workflow net in Figure~\ref{fig:workflow}.

%% file: workflow1.tex
\begin{tikzpicture}
		[place/.style={draw, circle}, 
		transition/.style={draw, minimum height=0.5cm},
		silent transition/.style={transition, fill},
		treeNode/.style={draw, dashed}]
		
		\node (source) [place] {};
		
		\def\firstone{5mm}
		\def\secondone{5mm};
		\def\milestoneplace{5mm};
		\def\dist{3.3mm};
		\node (int1) [silent transition, right of=source] {};
		
		\node (int10) [place, right of=int1, yshift=\firstone] {};
		\node (int11) [silent transition, right=\dist of int10] {};
		\node (int12) [place, right=\dist of int11, yshift=\secondone] {};
		\node (int112) [place, right=\dist of int11, yshift=-\secondone] {};
		\node (int13) [silent transition, right=\dist of int12] {};
		\node (int113) [silent transition, right=\dist of int112] {};
		\node (int14) [place, right=\dist of int13] {};
		\node (a) [transition, right=\dist of int14, label=above right:$a$] {};
		\node (int16) [place, right=\dist of a] {};
		\node (int17) [silent transition, right=\dist of int16] {};
		\node (int18) [place, right=\dist of int17] {};
		\node (int19) [silent transition, right=\dist of int18, yshift=-\secondone] {};
		\node (int20) [place, right=\dist of int19] {};
		
		\node (int21) [place, right=\dist of int113] {};
		\node (b) [transition, right=\dist of int21, label=above right:$b$] {};
		\node (int23) [place, right=\dist of b] {};
		\node (int117) [silent transition, right=\dist of int23] {};
		\node (int118) [place, right=\dist of int117] {};
		\node (milestone) [place] at ($(a)!0.5!(b)$) {};
		
		\node (c) [transition, below=2mm of b, label=above right:$c$] {};		
		\node (int4) [place, left of=c] {};
		\node (int6) [place, right of=c] {};
		
		\node (int9) [silent transition, right of=int20, yshift=-\firstone] {};
		
		\draw [->] (source) to (int1);
		\draw [->] (int1) to (int4);
		\draw [->] (int4) to (c);
		\draw [->] (c) to (int6);
		\draw [->] (int6) to (int9);
		\draw [->] (int1) to (int10);
		\draw [->] (int10) to (int11);
		\draw [->] (int11) to (int12);
		\draw [->] (int12) to (int13);
		\draw [->] (int13) to (int14);
		\draw [->] (int14) to (a);
		\draw [->] (a) to (int16);
		\draw [->] (int16) to (int17);
		\draw [->] (int17) to (int18);
		\draw [->] (int18) to (int19);
		\draw [->] (int19) to (int20);
		\draw [->] (int113) to (int21);
		\draw [->] (int112) to (int113);
		\draw [->] (int11) to (int112);
		\draw [->] (int11) to (milestone);
		\draw [->] (milestone) to (int13);
		\draw [->] (milestone) to (int113);
		\draw [->] (int21) to (b);
		\draw [->] (b) to (int23);
		\draw [->] (int23) to (int117);
		\draw [->] (int117) to (milestone);
		\draw [->] (int117) to (int118);
		\draw [->] (int118) to (int19);
		\draw [->] (milestone) to (int19);
		\draw [->] (int17) to (milestone);
		\draw [->] (int20) to (int9);
		
		\node (int) [treeNode, fit=(int10) (a) (int20) (b)] {};
		\node (con) [treeNode, fit=(int1) (a) (int9) (c)] {};
		
		\node (seq2) [place, above=3.1mm of a] {};
		\node (seq1) [transition, left of=seq2, label=above right:$d$] {};
		\node (seq3) [transition, right of=seq2, label=above right:$e$] {};
		
		\draw [->] (source) |- (seq1);
		\draw [->] (seq1) to (seq2);
		\draw [->] (seq2) to (seq3);
		
		\node (seq) [treeNode, fit=(seq1) (seq3)] {};
		
		\node (loop3) [transition, below=2.2mm of c, label=above right:$f$] {};
		\node (loop2) [place, left of=loop3] {};
		\node (loop1) [silent transition, left of=loop2] {};		
		\node (loop4) [place, right of=loop3] {};
		\node (loop5) [transition, below=1mm of loop3, label=above right:$g$] {};
		\node (loop6) [silent transition, right of=loop4] {};

		\draw [->] (source) |- (loop1);		
		\draw [->] (loop1) to (loop2);
		\draw [->] (loop2) to (loop3);
		\draw [->] (loop3) to (loop4);
		\draw [->] (loop4) |- (loop5);
		\draw [->] (loop5) -| (loop2);
		\draw [->] (loop4) to (loop6);
		
		\node (loop) [treeNode, fit=(loop1) (loop6) (loop5)] {};
		
		\node (sink) [place, right of=int9] {};
		\draw [->] (loop6) -| (sink);
		\draw [->] (seq3) -| (sink);
		\draw [->] (int9) to (sink);
		
		\node (tree) [treeNode, fit=(source) (loop5) (seq1) (sink)] {};
		
	\end{tikzpicture}

%% file: sec4rules.tex
	In this section, we restate the process tree reduction rules as formally defined in~\cite{kais}.
	In the following definitions, $M$ refers to a process tree and $\ldots$ refers to any number of process trees (possibly 0).
	
	\begin{definition}[Singularity]
	\label{def:rules:sing}
		Let $\op \in \{\xorOp, \sequenceOp, \concurrentOp, \interleavedOp, \orOp\}$.
		Then:
		\begin{align*}
			&(\ruleSin{}) & \op(M) \Rightarrow {}& M
		\end{align*}
	\end{definition}
	
	\begin{definition}[Associativity]
	\label{def:rules:assoc}
		\begin{align*}
			&(\ruleAssXor{}) & \xorOp ( \ldots_1, \xorOp(\ldots_2) ) \Rightarrow {}& \xorOp( \ldots_1, \ldots_2 ) \\
            &(\ruleAssSeq{}) & \sequenceOp ( \ldots_1, \sequenceOp(\ldots_2), \ldots_3 ) \Rightarrow {}& \sequenceOp( \ldots_1, \ldots_2, \ldots_3 )\\
                &(\ruleAssCon{}) & \concurrentOp ( \ldots_1, \concurrentOp(\ldots_2) ) \Rightarrow {}& \concurrentOp( \ldots_1, \ldots_2 )	\\
                &(\ruleAssOr{}) & \orOp(\ldots_1, \orOp(\ldots_2)) \Rightarrow{}& \orOp(\ldots_1, \ldots_2) \\
                &(\ruleAssLoopB{}) & \loopOp ( \loopOp( M, \ldots_1 ), \ldots_2 ) \Rightarrow {}& \loopOp( M, \ldots_1, \ldots_2 )\\
                &(\ruleAssLoopR{}) & \loopOp ( M, \ldots_1, \xorOp( \ldots_2 ) ) \Rightarrow {}& \loopOp( M, \ldots_1, \ldots_2 )
		\end{align*}
	\end{definition}
	
	\begin{definition}[$\tau$s: non-Behaviour-Adding]
	\label{def:rules:nba}
		\begin{align*}
			&(\ruleTauSeq{}) & \sequenceOp(\ldots, M, \tau) \Rightarrow {}& \sequenceOp(\ldots, M) \\
			&(\ruleTauCon{}) & \concurrentOp(\ldots, M, \tau) \Rightarrow {}& \concurrentOp(\ldots, M) \\
			&(\ruleTauInt{}) & \interleavedOp(\ldots, M, \tau) \Rightarrow {}& \interleavedOp(\ldots, M)\\
			&(\ruleTauLoopBR{}) & \loopOp (\tau, \tau) \Rightarrow {}& \tau
		\end{align*}
	\end{definition}
	
	\begin{definition}[$\tau$s: Double Skips]
	\label{def:rules:ds}
		Let $Q$ be a process tree such that $\epsilon \in \lang{Q}$.
		Then:
		\begin{align*}
			&(\ruleTauXor{}) & \xorOp(\ldots, Q, \tau)\ \Rightarrow {}& \xorOp(\ldots, Q)\\
			&(\ruleTauLoopR{}) & \loopOp(M, \ldots, Q, \tau)\ \Rightarrow {}& \loopOp(M, \ldots, Q)
		\end{align*}
	\end{definition}
	
	\begin{definition}[$\tau$s: Pulling-up]
	\label{def:rules:pu}
		Let $P$ be a process tree such that $\lang{P} \neq \{\epsilon\}$.
		Then:
		\begin{align*}
			&(\ruleTauOr{}) & \orOp(\ldots, M, \tau) \Rightarrow{}& \xorOp(\tau, \orOp(\ldots, M))\\
			&(\ruleTauOrXor{}) & \orOp(\ldots_1, \xorOp(\ldots_2, M, \tau)) \Rightarrow{}& \xorOp(\tau, \orOp(\ldots_1, \xorOp(\ldots_2, M)))\\
			&(\ruleTauLoopB{}) & \loopOp(\tau, \ldots, P) \Rightarrow {}& \xorOp(\tau, \loopOp(\xorOp(\ldots, P), \tau))
		\end{align*}
	\end{definition}
	
	\begin{definition}[Implicit $\concurrentOp$]
	\label{def:rules:con}
		Let $Q_1$, $Q_2$ be process trees such that $\epsilon \in \lang{Q_1}, \lang{Q_2}$, and let $S_1 \ldots S_n$ be process trees such no trace in the language of these trees has more than one event ($\forall_{1 \leq i \leq n, t \in \lang{S_i}} |t| \leq 1$).
		Then:
		\begin{align*}
                &(\ruleConInt{}) & \interleavedOp(S_1, \ldots S_n) \Rightarrow {}& \concurrentOp(S_1, \ldots S_n)\\
                &(\ruleConOr{}) & \concurrentOp(\ldots, Q_1, Q_2) \Rightarrow{}& \concurrentOp(\ldots, \orOp(Q_1, Q_2))
		\end{align*}
	\end{definition}

%% file: sec5processtrees.tex
In this section, we motivate and analyse the set of reduction rules.
We first introduce the rules in their categories, and illustrate them with examples.
Second, we prove that the set of rules is correct, terminating, confluent and complete.

\subsection{Motivation}
\label{sec:motivation}
	For ease of explanation, we identify four categories of rules: (1) singularity rules, targeting nodes with one child, (2) associativity rules, targeting nodes with a child of the same process tree operator, (3) $\tau$-reduction rules, and (4) concurrency reduction rules, targeting nodes that are concurrent but appear as different constructs.

	\fakesub{Singularity (Definition~\ref{def:rules:sing}).}
		The singularity rule $\ruleSin{}$ removes any subtree that has only one child.
		Consequently, it applies to any process tree operator except $\loopOp$, as $\loopOp$-nodes always have at least 2 children.
		For instance, the \ruleSin{} rule performs the reduction of Figure~\ref{fig:exa:sing} in two steps.
		
		\begin{figure}
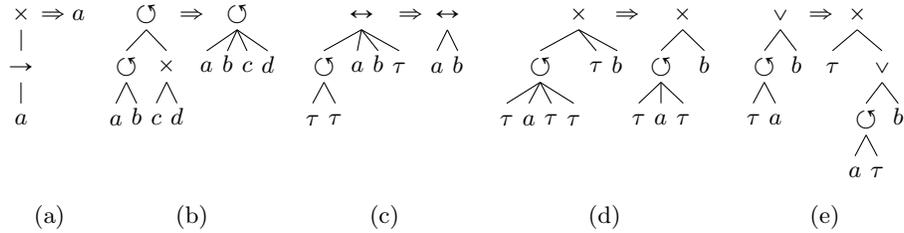

			\centering
			\begin{subfigure}[b]{0.09\linewidth}
				$\inlineTree{$\xorOp$}{ [.$\sequenceOp$ [.$a$ \edge[draw=none]; ~ ] ] } \Rightarrow a$
				\caption{}
				\label{fig:exa:sing}
			\end{subfigure}%
			\begin{subfigure}[b]{0.22\linewidth}
				\centering
				$\inlineTree{$\loopOp$}{ [.$\loopOp$ [.$a$ \edge[draw=none]; ~ ] $b$ ] [.$\xorOp$ $c$ $d$ ] } \hspace{-1.5mm} \Rightarrow \hspace{-1.5mm} \inlineTree{$\loopOp$}{$a$ $b$ $c$ $d$}$
				\caption{}
				\label{fig:exa:asso}
			\end{subfigure}%
			\begin{subfigure}[b]{0.20\linewidth}
				\centering
				$\inlineTree{$\interleavedOp$}{ [.$\loopOp$ $\tau$ [.$\tau$ \edge[draw=none]; ~ ] ] $a$ $b$ $\tau$ } \hspace{-2mm} \Rightarrow \inlineTree{$\interleavedOp$} {$a$ $b$ }$
				\caption{}
				\label{fig:exa:nba}
			\end{subfigure}%
			\begin{subfigure}[b]{0.28\linewidth}
				\centering
				$\inlineTree{$\xorOp$}{ [.$\loopOp$ $\tau$ [.$a$ \edge[draw=none]; ~ ] $\tau$ $\tau$ ] $\tau$ $b$ } \hspace{-2mm} \Rightarrow \hspace{-2mm} \inlineTree{$\xorOp$}{ [.$\loopOp$ $\tau$ $a$ $\tau$ ] $b$ }$
				\caption{}
				\label{fig:exa:ds}
			\end{subfigure}%
			\begin{subfigure}[b]{0.20\linewidth}
				\centering
				$\inlineTree{$\orOp$}{ [.$\loopOp$ $\tau$ $a$ ] $b$ } \Rightarrow \hspace{-2mm} \inlineTree{$\xorOp$}{$\tau$ [.$\orOp$ [.$\loopOp$ $a$ $\tau$ ] $b$ ] }$
				\caption{}
				\label{fig:exa:pu}
			\end{subfigure}
			\caption{Example reduction rule applications.}
		\end{figure}
	
	\fakesub{Associativity (Definition~\ref{def:rules:assoc}).}
		The associativity rules target nested nodes of the same operator.
		There are associativity rules for $\xorOp$, $\sequenceOp$, $\concurrentOp$ and $\orOp$, and two for $\loopOp$, targeting both first and non-first children.
		Figure~\ref{fig:exa:asso} shows an example reduction.
		Please note that in rule $\ruleAssSeq{}$, due to the non-symmetry of $\sequenceOp$, children might be present before and after the inner $\sequenceOp$-child.
		
		There is no associativity rule for the interleaved operator $\interleavedOp$, as this operator is not associative, that is, the nesting of $\interleavedOp$ nodes matters.
		For instance, the trees $\interleavedOp(a, b, c)$ and $\interleavedOp(\interleavedOp(a, b), c$) do not have the same language, as the first tree expresses the trace $\langle a, c, b \rangle$ but the second tree does not.
		
	\fakesub{$\tau$-Reduction.}
		A source of complexity in process trees are silent steps ($\tau$-leaves).
		In this section, we identify three types of reduction rules for $\tau$ leaves: (1) rules that target $\tau$ leaves that do not add anything to the language of a tree, (2) rules that target duplicated skipping constructs ($\xorOp(\tau,.)$), and (3) rules that pull skips up higher in the tree.
	
		First, a $\tau$ leaf has no influence on the language if that leaf is a child of a $\sequenceOp$, $\concurrentOp$ or $\interleavedOp$ node.
		Hence, such leaves can be removed without altering the language (Definition~\ref{def:rules:nba}).
		Also, if all children of a node are $\tau$-leaves, then the node itself can be replaced by a $\tau$, which can be achieved for the aforementioned process tree operators by rule $\ruleSin{}$.
		As $\ruleSin{}$ does not apply to $\loopOp$, we need a separate rule to transform $\loopOp(\tau,\tau)$-constructs into $\tau$-leaves ($\ruleTauLoopBR{}$).
		Please note that in these rules, $\tau$ children are only removed if they are not the only child of their parent, preventing the creation of operator nodes without children.		
		Figure~\ref{fig:exa:nba} shows an example reduction.
		
		Second, an $\xorOp$ node with a $\tau$ child expresses that the node can be skipped. 
		That is, execution of the $\xorOp$ node might not result in any visible execution of an activity.
		However, multiple children of the $\xorOp$ node might be able to produce an empty trace.
		Then, the $\tau$ child does not add anything to the language of the tree and can be removed (Definition~\ref{def:rules:ds}).
		This also holds for non-first children of $\loopOp$ nodes.
		An example reduction is shown in Figure~\ref{fig:exa:ds}.
		Please note that we defined $Q$ to contain the empty trace.
		This can be easily checked using a recursive function based on the structure of process trees; a state-space exploration is not necessary (for details, see \appRef{app:treeEmptyTrace}).
		
		Third, we consider more complex $\tau$-reduction rules, regarding the ``skip'' construct $\xorOp(\tau, .)$: these constructs can be pulled up in the hierarchy of the process tree (Definition~\ref{def:rules:pu}).
		An example reduction using these rules is shown in Figure~\ref{fig:exa:pu}.
		Next, we explain design decisions and particularities of these rules.
						
		\paragraph{Pulling Up vs. Pushing Down.}
			Intuitively, a skip construct can be pushed down or pulled up in a tree.
			For instance, consider the trees $\xorOp(\tau, \orOp(a,  b))$ and $\orOp(a, b, \tau)$.
			Both process trees express the same language.
			However, in the first tree, the skip construct is at the top of the tree ($\xorOp(\tau,.)$), while in the second tree the skip construct is at a position lower in the tree and is expressed differently ($\tau$).
		
			Thus, in some cases, one can choose to either pull skips up in the tree or push skips down in the tree.
			Pulling skips up in the tree is deterministic (as each tree node only has one parent), while pushing skips down is not deterministic (it might be possible to push skips to \emph{every} child independently).
			Such nondeterminism prevents confluence and completeness, as it might not be clear to \emph{which} child a skip construct should be pushed down.
			Therefore, our reduction rules pull skip constructs up in the tree.
	
		\paragraph{Increasing Number of Nodes.}
			On first sight, the use of these reduction rules might not be obvious as they \emph{increase} the number of nodes in the tree.
			However, in some cases a temporary increase in nodes might be necessary to achieve a smaller process tree later on.
			For instance, consider the following sequence of process trees, each obtained by applying a reduction rule to the previous one:
			$$
				\inlineTree{$\xorOp$}{ [.$\orOp$ $a$ $\tau$ ] [.$\orOp$ $b$ $\tau$ ] } 
				\xrightarrow{\ruleTauOr{} (\times 2)} 
				\inlineTree{$\xorOp$}{ [.$\xorOp$ $\tau$ [.$\orOp$ $a$ ] ] [.$\xorOp$ $\tau$ [.$\orOp$ $b$ ] ] }
				\xrightarrow{\ruleSin{} (\times 2)}
				\inlineTree{$\xorOp$}{ [.$\xorOp$ $a$ $\tau$ ] [.$\xorOp$ $b$ $\tau$ ] }
				\xrightarrow{\ruleAssXor{} (\times 2)}
				\inlineTree{$\xorOp$}{ $a$ $\tau$ $b$ $\tau$ }
				\xrightarrow{\ruleTauXor{}}
				\inlineTree{$\xorOp$}{$a$ $b$ $\tau$}
			$$
			\noindent
			Without rule \ruleTauOr{}, the first tree cannot be reduced and would be in normal form.
			With this rule, the size of the tree is temporarily increased, however afterwards the tree can be reduced further by other rules.
			In Section~\ref{sec:analysis}, we will show that even though \ruleTauOr{} increases the number of nodes, the set of rules still cannot be applied endlessly.
	
	\fakesub{Implicit $\concurrentOp$ (Definition~\ref{def:rules:con}).}
		The last rules relate to the concurrency operator $\concurrentOp$: some $\orOp$ and $\interleavedOp$ constructs might represent the same language as a comparable $\concurrentOp$ construct.
		
		Please note that $\ruleConInt{}$ requires a check whether the children $S_1 \ldots S_1$ can only produce traces that have length zero or one.
		A structure-based function for this is available in \appRef{app:shortTraces}.		
		Also note that due to the non-associativity of $\interleavedOp$, the rule $\ruleConInt{}$ can only be applied to \emph{all} children of a $\interleavedOp$ node at the same time.
		For instance, the trees $\interleavedOp(a, b, \sequenceOp(c, d))$ and $\interleavedOp(\concurrentOp(a, b), \sequenceOp(c, d))$ have a different language.
		
		\paragraph{Increasing Number of Nodes.}
			These rules increase the number of nodes, but have the potential to decrease the size of trees in combination with other rules.
			Consider the following sequence of trees, in which rule \ruleConOr{} is necessary to reduce the first tree, to which no other rules can be applied:
			$$
				\inlineTree{$\concurrentOp$}{[.$\xorOp$ $\tau$ $a$ ] [.$\xorOp$ $\tau$ $b$ ] }
				\xrightarrow{\ruleConOr{}}
				\inlineTree{$\concurrentOp$}{[.$\orOp$ [.$\xorOp$ $\tau$ $a$ ] [.$\xorOp$ $\tau$ $b$ ] ] }
				\xrightarrow{\ruleSin}
				\inlineTree{$\orOp$}{ [.$\xorOp$ $\tau$ $a$ ] [.$\xorOp$ $\tau$ $b$ ] }
				\xrightarrow{\ruleTauOr{} (\times 2)}
				\inlineTree{$\xorOp$}{ $\tau$ [.$\xorOp$ $\tau$ [.$\orOp$ [.$\xorOp$ $a$ ] [.$\xorOp$ $b$ ] ] ] }
				\xrightarrow{\ruleAssXor, \ruleSin{}~(\times 2)}
				\inlineTree{$\xorOp$}{ $\tau$ $\tau$ [.$\orOp$ $a$ $b$ ] }
				\xrightarrow{\ruleTauXor}
				\inlineTree{$\xorOp$}{ $\tau$ [.$\orOp$ $a$ $b$ ] }
			$$

			\vspace{-4mm}
		\paragraph{Process Trees vs. Petri Nets.}
			All reduction rules aim to decrease the number of nodes in a process tree.
			However, some use cases require the tree to be translated to a Petri net~\cite{DBLP:books/sp/Aalst16}, thus shifting the intended minimisation from the size of the process tree to the size of the translated Petri net.
			Intuitively, minimising the number of nodes in a process tree also minimises the number of nodes in a corresponding translated Petri net, except for the operators $\interleavedOp$ and $\orOp$, as typical translations of these operators involve many Petri net elements~\cite{kais}, thus introducing $\interleavedOp$ and $\orOp$ nodes typically increases complexity in a translated Petri net.
			Rule $\ruleConInt{}$ helps in this regard as it removes an $\interleavedOp$ node and adds a $\concurrentOp$ node, which decreases Petri net complexity.
			However, rule $\ruleConOr{}$ replaces a $\concurrentOp$ with an $\orOp$ node, which increases Petri net complexity.				
			Therefore, if a use case asks for a smallest Petri net, one could consider excluding or reversing the rule $\ruleConOr{}$, however the influence on the formal results is outside the scope of this paper.
    	
\subsection{Analysis}
\label{sec:rediscoverability}
\label{sec:analysis}
	In this section, we analyse the reduction rules with respect to the four properties, described in Section~\ref{sec:introduction}, that ideally each set of rules should adhere to: correctness, termination, confluence and completeness.
	Furthermore, we discuss limitations of the set of rules, compare them to the Petri net rules of~\cite{DBLP:conf/bpm/PedroC16}, and discuss their implementation.

	\fakesub{Correctness.}
		First, rules need to preserve the language of the process trees they are applied to.
		As argued in Section~\ref{sec:motivation} and by the semantics of process trees:
		\begin{lemma}
			The reduction rules preserve language.
		\end{lemma}

	\fakesub{Termination.}
		In order to show termination, our proof strategy is to give a function $\varphi : \text{\textit{tree}} \rightarrow \mathbb{N}^+$ that decreases with every rule application.
		That is, before the application of any rule, the function applied to the tree has a higher value than after the application.
		As discussed earlier in this section, some process tree reduction rules increase the number of nodes in a process tree.
		Therefore, a function that simply counts the number of nodes in a tree would not suffice and a more elaborate function is necessary.
		Instead, we use the following helper functions, all of which take a process tree $M$, and let $M'$ be any subtree of $M$:\\
		\begin{tabular}{lp{11cm}}
			$N$ & The number of \underline{n}odes: $N = |\{ M' \in M \}|$. \\
			$C_{\op}$ & number of times a node with operator $\op$ has a direct or indirect $\tau$ \underline{c}hild: $ C_{\op} = |\{ (M'_p, M'_c) \in M \mid M'_c = \tau \land M'_c \text{ is a direct or indirect child of } M'_p \}| $.\\
			$LBE$ & The number of $\loopOp$ nodes such that the \underline{l}oop \underline{b}ody can produce the \underline{e}mpty trace: $LBE = |\{ M' \in M \mid M' = \loopOp(M'_1, \ldots) \land \epsilon \in \lang{M'_1} \}|$.\\
			$O_{\op}$ & The number of $\op$ \underline{o}perator nodes: $ O_{\op} = |\{ M' \in M \mid M' = \op(\ldots)\}|$.\\
			$P_{\op}$ & The number of nodes that have a $\op$ node as \underline{p}arent: $ P_{\op} = |\{ M' \mid \op(\ldots, M', \ldots) \in M \}|$.\\
		\end{tabular}
		
		For instance:
		\begin{align*}
			N(\xorOp(\sequenceOp(a, b), c)) ={}& 5\\
			C_{\orOp}(\orOp(\orOp(\tau, \tau))) ={}& 4\\
			LBE(\loopOp(\loopOp(\loopOp(\tau, a), b), c)) ={}& 3\\
			O_{\interleavedOp}(\interleavedOp(\interleavedOp(a, b), c)) ={}& 2\\
			P_{\concurrentOp}(\concurrentOp(\concurrentOp(a, b), c)) ={}& 4
		\end{align*}
           
		Using these helper functions, we can choose our function $\varphi$: let $k$ be a constant and $M$ be a process tree, then
		$\varphi(k, M) = N(M) + C_{\orOp}(M) k + LBE(M) k^3 + O_{\interleavedOp}(M) k^4 + P_{\concurrentOp}(M) k^5$.
		Then, we can show termination:
	
		\begin{lemma}
			\label{lem:fundamentals:reductionRulesTerminating}
			Repeated application of the reduction rules is terminating.
		\end{lemma}
		\begin{proof}
			Let $M$ be a process tree, and let $k$ be sufficiently large.
			First, we show that $\varphi(k, M)$ decreases with each rule application:
			\begin{itemize}
				\item $\ruleTauOr{}$: this rule increases the number of nodes ($N$) by 1.
					However, it decreases the appearance of $\tau$s under $\orOp$-nodes ($C_{\orOp}$) by 1.
					Then, by construction of $\varphi$, for sufficiently large $k$, $\ruleTauOr{}$ decreases $\varphi$.
				\item $\ruleTauOrXor{}$: similarly, this rule increases $N$ by 1 and decreases $C_{\orOp}$ by 1.
				\item $\ruleTauLoopB{}$: similarly, this rule increases $N$ by 3 and decreases $LBE$ by 1.
				\item $\ruleConOr{}$: similarly, this rule increases $N$ by 1, increases $C_{\orOp}$ by at most $N$ and decreases $P_{\concurrentOp}$ by 1.
				As $C_{\orOp} \leq N^2$, this rule decreases $\varphi$.
			\end{itemize}
			(The other rules decrease all helper functions).
			Therefore, for $k \geq N(M)$, $\varphi(k, M)$ decreases with every rule application.
			Hence, applying the reduction rules repeatedly is terminating.\qed
		\end{proof}
		
	\fakesub{Confluence.}
		If a set of reduction rules is confluent, then applying the rules in a repeated fashion always reduces a particular tree to the same end result.
		Consequently, it does not matter what rules are applied or in which order the rules are applied: the end result will always be the same.
		A weaker notion is \emph{local confluence}: a set of rules is locally confluent if for every situation where two rules can be applied to a process tree and the possibilities deviate, it is possible to reduce these deviating possibilities back to a single common process tree~\cite{newman1942theories}.
		Our proof strategy is to show that the set of rules presented in this paper is locally confluent, after which we can conclude confluence.
		
		\begin{lemma}
		\label{lem:reductionRulesLocallyConfluent}
			The reduction rules are locally confluent.
		\end{lemma}
		We prove this lemma by considering all pairs of rules that might overlap (that is, apply to the same nodes), and showing that a common model can still be obtained.
		\begin{proof}
			We show the local confluence property for each pair of rules, excluding the trivial cases where the left-hand sides of the rules do not overlap.
			\begin{itemize}
				\item \ruleSin{} overlaps with \ruleTauOrXor{} if $\ldots_1$ in \ruleTauOrXor{} is empty.
					After applying \ruleTauOrXor{}, a common model can be reached by applying \ruleSin{} and \ruleAssXor{}.
					A similar argument holds for \ruleSin{} and \ruleConInt{}, which overlap if $n=1$.
				\item \ruleAssXor{} overlaps with \ruleAssLoopR{} if $\ldots_2$ contains an $\xorOp$ node.
					After applying \ruleAssLoopR{}, a common model is reached with applying \ruleAssLoopR{} another time.
				\item \ruleAssOr{} overlaps with \ruleTauOrXor{} if $\ldots_1$ in the left hand side of \ruleAssOr{} is the left hand side of \ruleTauOrXor{} (or symmetrically).
					After applying \ruleAssOr{}, a common model is reached with applying \ruleTauXor{}, \ruleSin{} and \ruleTauOrXor{} repeatedly.
					A similar argument holds for \ruleTauOr{}.
				\item \ruleAssLoopR{} overlaps with \ruleTauLoopB{} if $\ldots$ of the latter contains an $\xorOp$ child.
					After applying \ruleTauLoopB{}, a common model is reached by applying \ruleAssXor{} (on that tree) or \ruleAssLoopR{} and \ruleTauLoopB{} (on the original tree).
				\item \ruleTauXor{} overlaps with \ruleTauOrXor{} if $\ldots_2$ of the latter contains a $Q$-subtree.
					After applying \ruleTauXor{}, a common model is reached by applying \ruleTauOrXor{}, \ruleTauXor{} and \ruleSin{}.
				\item \ruleTauCon{} overlaps with \ruleConOr{} if one of the $Q$'s is a $\tau$.
					Applying \ruleTauCon{} and applying \ruleConOr{}, \ruleTauOr{} and \ruleSin{} yields the same model.
				\item \ruleTauInt{} overlaps with \ruleConInt{}, but applying \ruleTauCon{} yields a common model in a single step.
				\item \ruleTauLoopB{} overlaps with \ruleTauLoopR{}, but \ruleTauXor{} yields a common model in a single step.
				\item \ruleConOr{} overlaps with itself, but \ruleAssOr{} yields a common model.
			\end{itemize}
			Hence, the set of rules is locally confluent.\qed
		\end{proof}
		
		Newman's Lemma~\cite{newman1942theories} states that a set of reduction rules (a system of rewriting rules) that is locally confluent and terminating is confluent, thus our set is confluent.

		\begin{corollary}
		\label{cor:canonical}
			By Lemma~\ref{lem:fundamentals:reductionRulesTerminating}, Lemma~\ref{lem:reductionRulesLocallyConfluent} and Newman's Lemma~\cite{newman1942theories}, the reduction rules are confluent.
		\end{corollary}
		
		As a practical consequence, if there are several possibilities of applying a rule to a particular process tree, it does not matter which possibililty is chosen or in which order the possibilities are applied: the tree resulting of an exhaustive repeated application of the rules will always be the same.
		
	\fakesub{Completeness.}
		A set of reduction rules is complete with respect to a certain class $C$ of process models if any two language-equivalent models from $C$ are isomorphic after exhaustive application of the rules.
		For our rules, this class $C$ includes the constructs $\tau$, $\orOp$ and $\interleavedOp$, but excludes certain nestings of $\interleavedOp$ and $\loopOp$ nodes with $\tau$ children, and multiple occurrences of activities.
		For a full specification of $C$, please refer to~\appRef{app:completeness}.
		
		\begin{lemma}
		\label{lem:completeness}
			The rules are complete with respect to $C$.
			That is: let $P_1, P_2 \in C$ be process trees such that $\lang{P_1} = \lang{P_2}$ to which the rules have been applied exhaustively.
			Then, $P_1 = P_2$.
		\end{lemma}
		\begin{proof}
			Towards contradiction, 	assume that $P_1 \neq P_2$.
			By Lemma~50 of~\cite{kais}, there is a language abstraction $\alpha$ such that $\alpha(\lang{P_1}) \neq \alpha(\lang{P_2})$.
			Hence, $\lang{P_1} \neq \lang{P_2}$.
			\qed
		\end{proof}
	
		Outside the class $C$, our rules might not be strong enough to reduce a pair of language-equivalent trees to the same normal form.
		For instance, to each of the following pairs of language equivalent trees, no reduction rules can be applied, but still they are not isomorphic:
		\begin{align*}
                \loopOp(\sequenceOp(\xorOp(\tau, a), \loopOp(b, \tau)), \tau) \xLeftrightarrow{\text{future work}}{} & \loopOp(\sequenceOp(\xorOp(\tau, a), b), \tau)\\
                \concurrentOp(a, a) \xLeftrightarrow{\text{future work}}{}& \sequenceOp(a, a)\\
                \xorOp(a, a) \xLeftrightarrow{\text{future work}}{}& a\\
                \xorOp(\sequenceOp(a, b), \sequenceOp(b, a)) \xLeftrightarrow{\text{future work}}{}& \concurrentOp(a, b)
		\end{align*}
		
		In the first pair of trees, the challenge is that the inner loop has a $\tau$ as a non-first child and no reduction rule applies.
		For this example, a new rule would have to ``look'' several layers of process tree nodes deep, making it costly to execute.
		In the other pairs, activities appear more than once, which is not targeted by the rules of this paper. 
		In order to define reduction rules for these examples, one would need to perform a behavioural analysis of the process tree to find possible reductions, again making them costly to execute compared to the structural analysis based rules presented in this paper.
		
	\fakesub{Petri Net Rules.}
		To illustrate the difference in scope between the Petri net rules of~\cite{DBLP:conf/bpm/PedroC16} and the rules of this paper, we consider our process tree \inlineTree{$\xorOp$}{[.$\concurrentOp$ [.$\interleavedOp$ $a$ $b$ ] $c$ ] [.$\sequenceOp$ $d$ $e$ ] [.$\loopOp$ $f$ $g$ ] } and the translated workflow net shown in Figure~\ref{fig:workflow} again.		
		In isolation, the rules of~\cite{DBLP:conf/bpm/PedroC16} reduce this workflow net to the net shown in Figure~\ref{fig:workflow2}.
		In this net, the place $p_1$ could be removed without changing the language of the net, however there is no rule in~\cite{DBLP:conf/bpm/PedroC16} that targets this place.
		Such a redundant \emph{milestone} place makes the net non-free choice and is challenging for any Petri net reduction rule, as this redundancy is hard to establish in a structural rule.
		Thus, a state space exploration might be necessary to identify the redundancy of $p_1$, rather than a simple structured rule.		
		If instead we reduce the process tree using the reduction rules of this paper, we obtain \inlineTree{$\xorOp$}{[.$\concurrentOp$ $a$ $b$ $c$ ] [.$\sequenceOp$ $d$ $e$ ] [.$\loopOp$ $f$ $g$ ] } using reduction rules \ruleConInt{} and \ruleAssCon{}.
		The translation of this latter process tree yields a workflow net identical to Figure~\ref{fig:workflow2}, but without place $p_1$.
		Thus, as the language relation between activities is clear in trees, complex constructs (e.g. $\interleavedOp$ and $\orOp$) are reducible using a structural reduction rule.
			
		\begin{figure}
			\centering
			\input{workflow2others}
			\caption{The workflow net of Figure~\ref{fig:workflow}, reduced using~\cite{DBLP:conf/bpm/PedroC16}.
				Using the rules of this paper, place $p_1$ can be avoided.}
			\label{fig:workflow2}
		\end{figure}
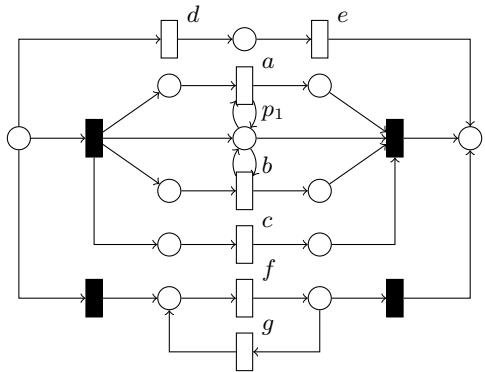
		
	\fakesub{Implementation.}
		The reduction rules introduced in this paper have been implemented as a plug-in of the ProM framework~\cite{DBLP:conf/apn/DongenMVWA05} and are extensible.
		The plug-in uses a brute-force approach on rules and nodes until no rule can be applied anymore.
		An interesting area of further research would be to study strategies to optimise rule applications, even though in our experiments reduction time was negligible (less than a second).

%% file: workflow2others.tex
\begin{tikzpicture}
		[place/.style={draw, circle}, 
		transition/.style={draw, minimum height=0.5cm},
		silent transition/.style={transition, fill},
		treeNode/.style={draw, dashed}]
		
		\node (source) [place] {};
		
		\def\firstone{7mm}
		\def\secondone{7mm};
		\def\milestoneplace{5mm};
		\def\dist{3.3mm};
		\node (int1) [silent transition, right of=source] {};
		
		\node (int12) [place, right of=int1, yshift=\secondone] {};
		\node (int112) [place, right of=int1, yshift=-\secondone] {};
		\node (a) [transition, right of=int12, label=above right:$a$] {};
		\node (int16) [place, right of=a] {};
		
		\node (b) [transition, right of=int112, label=above right:$b$] {};
		\node (int23) [place, right of=b] {};
		\node (milestone) [place, label=above right:$p_1$] at ($(a)!0.5!(b)$) {};
		
		\node (c) [transition, below=2mm of b, label=above right:$c$] {};		
		\node (int4) [place, left of=c] {};
		\node (int6) [place, right of=c] {};
		
		\node (int9) [silent transition, right of=int16, yshift=-\firstone] {};
		
		\draw [->] (source) to (int1);
		\draw [->] (int1) |- (int4);
		\draw [->] (int4) to (c);
		\draw [->] (c) to (int6);
		\draw [->] (int6) -| (int9);
		\draw [->] (int1) to (int12);
		\draw [->] (int12) to (a);
		\draw [->] (a) to (int16);
		\draw [->] (int16) to (int9);
		\draw [->] (int112) to (b);
		\draw [->] (int1) to (int112);
		\draw [->] (int1) to (milestone);
		\draw [->, bend left] (milestone) to (a);
		\draw [->, bend left] (milestone) to (b);
		\draw [->, bend left] (a) to (milestone);
		\draw [->, bend left] (b) to (milestone);
		\draw [->] (b) to (int23);
		\draw [->] (int23) to (int9);
		\draw [->] (milestone) to (int9);
		
		\node (seq2) [place, above=2mm of a] {};
		\node (seq1) [transition, left of=seq2, label=above right:$d$] {};
		\node (seq3) [transition, right of=seq2, label=above right:$e$] {};
		
		\draw [->] (source) |- (seq1);
		\draw [->] (seq1) to (seq2);
		\draw [->] (seq2) to (seq3);
		
		\node (loop3) [transition, below=2mm of c, label=above right:$f$] {};
		\node (loop2) [place, left of=loop3] {};
		\node (loop1) [silent transition, left of=loop2] {};		
		\node (loop4) [place, right of=loop3] {};
		\node (loop5) [transition, below=2mm of loop3, label=above right:$g$] {};
		\node (loop6) [silent transition, right of=loop4] {};

		\draw [->] (source) |- (loop1);		
		\draw [->] (loop1) to (loop2);
		\draw [->] (loop2) to (loop3);
		\draw [->] (loop3) to (loop4);
		\draw [->] (loop4) |- (loop5);
		\draw [->] (loop5) -| (loop2);
		\draw [->] (loop4) to (loop6);
		
		\node (sink) [place, right of=int9] {};
		\draw [->] (loop6) -| (sink);
		\draw [->] (seq3) -| (sink);
		\draw [->] (int9) to (sink);
		
	\end{tikzpicture}

%% file: sec6evaluation.tex
In this section, we illustrate the practical applicability of the reduction rules using a real-life experiment, mimicking the setting of a user applying a process discovery technique and aiming for a model that is as small as possible (for human understanding and further machine processing).
That is, we show the size reduction achieved by the Petri net reduction rules of~\cite{DBLP:conf/bpm/PedroC16} (PN), the block-structured-Petri-net reduction rules presented in this paper (PT), and the combination of the two sets of rules (PT\&PN).
	
To this end, we apply three process discovery techniques that return process trees to 16 publicly available real-life event logs.
As two of the discovery techniques are non-deterministic, the techniques are applied 5 times to each event log.
For each such obtained process tree, four sets of size measures are taken: the returned model without reduction, reduced with PT, reduced with PN, and reduced with PT\&PN.
For each of these sets, if applicable, the size of the process tree (that is, the number of nodes in the tree) and the size of the Petri net (that is, the total number of nodes, transitions and edges) are reported.
	
The discovery techniques that were used were the deterministic Inductive Miner - infrequent (IMf)~\cite{DBLP:conf/bpm/LeemansFA13}, the genetic Evolutionary Tree Miner (ETM)~\cite{DBLP:conf/cec/BuijsDA12} and the Indulpet Miner (IN)~\cite{DBLP:conf/otm/LeemansTH18}\footnote{The implementations of IMf and IN apply a subset of the rules in this paper by default. ETM applies some associativity rules. For this experiment, all these reduction steps were disabled to show the impact of the reduction rules.}
We provided a time limit of one hour to the techniques and applied them to the following event logs: a log of a mortgage-application process (BPIC12), 5 logs of building-permit application processes in different municipalities (BPIC15), 
8 logs of agricultural-subsidy applications (BPIC18), a log of a road fine collection process (Roadfines), and a log of a sepsis treatment process in a hospital (Sepsis)\footnote{All datasets are available from \url{https://data.4tu.nl/repository/collection:event_logs_real}.}.

\begin{table}
	\caption{Results: size of resulting trees (PT) and Petri nets (PN).}
	\label{tbl:results}
	\centering
	\input{experiments/computer-discovered/results2.txt}
\end{table}

\paragraph{Results.}
	Table~\ref{tbl:results} shows the results.
	Most models by discovered by IMf are susceptible to reduction: all models except BPIC18-2 were made smaller by the reductions.
	The tree reduction (PT) by itself resulted in the smallest model once, the Petri net reduction (PN) by itself 5 times, while the combined reduction (PT\&PN) yielded the smallest models 9 times.
	The best result was achieved on BPIC18-4, where the PT\&PN-reduced models were 67\% of the size of the non-reduced models.

	ETM did not return models for the majority of the event logs in this experiment within the time we had available.
	For all of the models obtained by ETM, the combination of PT\&PN achieved smaller models than PT or PN in isolation.
	
	Indulpet Miner (IN) combines ETM and IMf and, in general, discovers smaller models than ETM or IMf.
	On these models, the PN reduction by itself was the most effective for 12 logs, while the PT reduction by itself was that in 3 cases.
	The combined PT+PN reduction resulted in the smallest models in 4 cases.
	Remarkably, for two logs (BPIC12 and BPIC18-7), the PN reduction by itself outperformed PT+PN (by 1 and 17 elements).
	A manual inspection revealed that this was due to rule \ruleTauLoopB{}, which illustrates the trade offs necessary between confluence, smaller process trees and smaller Petri nets.
	
	To conclude, for most of the combination of event logs and miners, the combination of PT\&PN outperformed the PT and PN reduction rules applied in isolation.
	
\paragraph{Reproducibility.}
	All event logs used in this experiment are publicly available; the code to execute the experiment itself is available at {\url{https://svn.win.tue.nl/repos/prom/Packages/SanderLeemans/Trunk/src/sosym2020}}.
	The machine on which the experiments were run possesses a Xeon 3.5GHz processor, 32GB of RAM and Windows 7, which was fully patched.

%% file: experiments/computer-discovered/results2.txt
\begin{tabular}{llrrrrrr}
\toprule
&&\multicolumn{2}{c}{not reduced}&\multicolumn{2}{c}{PT}&\multicolumn{1}{c}{PN}&\multicolumn{1}{c}{PT\&PN}\\
\cmidrule(lr){3-4}
\cmidrule(lr){5-6}
\cmidrule(lr){7-7}
\cmidrule(lr){8-8}

&&tree&pn&tree&pn&pn&pn\\
\midrule

BPIC12 & IMf&73&194&67&191&174&171\\
& IN&26&72&22&67&57&58\\

BPIC15-1 & IMf&998&2883&930&2635&2718&2579\\
&IN&14&45&13&45&45&45\\

BPIC15-2&IMf&615&1760&556&1588&1645&1548\\
&IN&45&133&40&133&129&129\\

BPIC15-3&IMf&712&1985&646&1809&1874&1761\\
&IN&23&58&21&58&52&52\\

BPIC15-4&IMf&595&1608&545&1520&1487&1431\\
&IN&124&379&115&379&352&352\\

BPIC15-5&IMf&649&1901&585&1682&1758&1598\\
&IN&91&292&84&289&277&274\\

Roadfines&IMf&34&102&31&94&98&94\\
&ETM&18&99&17&87&91&85\\
&IN&34&102&31&94&98&94\\

BPIC18-1&IMf&19&59&18&59&51&51\\
&ETM&38&312&36&285&265&253\\
&IN&19&59&18&59&51&51\\

BPIC18-2&IMf&13&37&12&37&37&37\\
&ETM&41&351&42&341&335&329\\
&IN&13&37&12&37&37&37\\

BPIC18-3&IMf&36&102&34&98&78&78\\
&IN&23&67&20&61&50&50\\

BPIC18-4&IMf&35&132&27&93&101&89\\
&IN&20&62&18&58&49&49\\

BPIC18-5&IMf&54&155&52&155&135&135\\
&IN&56&192&52&189&159&156\\

BPIC18-6&IMf&31&98&27&90&86&82\\
&IN&31&98&27&90&86&82\\

BPIC18-7&IMf&71&216&64&200&168&160\\
&IN&69&260&66&248&196&213\\

BPIC18-8&IMf&18&59&17&59&51&51\\
&IN&18&59&17&59&51&51\\

Sepsis&IMf&36&129&33&129&109&109\\
&IN&36&129&33&129&109&109\\

\bottomrule
\end{tabular}

%% file: sec7conclusion.tex
With the increase of data available in business processes, larger and more complex process models challenge human analysis and machine processing.
Reduction rules decrease the size and complexity of models while they preserve the language of the model, which makes these models easier to read and faster to process.
In this paper, we considered four desirable properties for such reduction rules: correctness, that is the rules preserve the language, termination, confluency, that is the order in which the rules are applied exhaustively is not relevant, and completeness, that is for a certain class of models it holds that two models of the class are language equivalent if and only if their reduced counterparts are syntactically equivalent.

We introduced a set of reduction rules for abstract hierarchical views of block-structured workflow nets (process trees), consisting of a singularity rule, associativity rules, several types of $\tau$-reduction rules and concurrency rules.
This set of rules was shown to be correct, terminating, confluent and complete.

An implementation has been provided, and our evaluation showed that the reduction rules introduced in this paper improve, that is, result in smaller and less complex models, compared to reducing translated Petri nets directly.
In particular, in many cases the best results were obtained using a combination of process tree and Petri net reduction, which illustrates that structural reduction on a higher level of abstraction (process trees) may target different properties of the language of a model than reduction on a lower level of abstraction (Petri nets).

As future work, we would suggest research into the completeness property, that is, for which classes of models it holds that two models have the same language if and only if their exhaustively reduced counterparts are syntactically equivalent.
For the block-structured workflow net rules of this paper, this class could be extended, while for the Petri net rules of~\cite{DBLP:conf/bpm/PedroC16}, such a class could be established.

Furthermore, the reduction rules of this paper and~\cite{DBLP:conf/bpm/PedroC16} could be extended with rules that target multiple occurrences of activities, such as the model $\xorOp(a, a, b)$, preferably with structure-based rules as far as possible.
Finally, other formalisms, such as BPMN models, EPCs and Petri nets, might benefit from a chain of steps consisting of a translation to block-structured workflow nets, then applying the rules of this paper, then translating the resulting model back to the original formalism.

\noindent\textit{Acknowledgement.}
We wish to thank Dirk Fahland and Wil van der Aalst for their input on and discussions about preliminary versions of the reduction rules.
